\documentclass{article}

\usepackage{times}
\usepackage{soul}
\usepackage{url}
\usepackage[hidelinks]{hyperref}
\usepackage[utf8]{inputenc}
\usepackage{graphicx}
\usepackage{amsmath}
\usepackage{amsthm}
\usepackage{booktabs}

\usepackage{tikz}
\usepackage{xspace}
\usepackage{dsfont}
\usepackage{algorithm}
\usepackage[noend]{algpseudocode}
\usepackage{xcolor}
\usepackage{multirow}
\usepackage{blindtext}
\usepackage{array}
\usepackage{subcaption}
\usepackage{xcolor}
\usepackage{todonotes}
\usetikzlibrary{automata,positioning}
\usetikzlibrary{shapes}
\usepackage{pdfpages}
\usepackage{paralist}
\usepackage{wasysym}

\newtheorem{example}{Example}
\newtheorem{theorem}{Theorem}

\title{Conflict-Based Search for Connected Multi-Agent Path Finding}

 \author{
 	Arthur Queffelec$^1$
 	\and
 	Ocan Sankur$^2$\and
 	Fran\c{c}ois Schwarzentruber$^1$
 	\\
 	$^1$Univ Rennes, CNRS, IRISA\\
 	$^2$Univ Rennes, CNRS, INRIA, IRISA\\
 	firstname.lastname@inria.fr
 }

\newcommand{\algorithmfont}{\small}

\newtheorem{definition}{Definition}

\algrenewcommand\algorithmicindent{0.6em}%

\newcommand{\coms}{\xspace\scalebox{0.5}{\begin{tikzpicture}\draw[dash pattern=on 3pt off 3pt, line width = 1mm ] (0,0) -- (0.8,0); \node at (0,-0.08) {}; \end{tikzpicture}}\xspace}
\renewcommand{\gets}{:=}

\newcommand{\set}[1]{\{#1\}}

\newcommand\sTopoGraph{G\xspace}

\newcommand\sVertexSet{V\xspace}
\newcommand\sBase{B\xspace}

\newcommand\sMoveRel{E_m}
\newcommand\sCommRel{E_c}

\newcommand\insCommRel[2]{(#1, #2) \in \sCommRel}

\newcommand\sConstTree{CT}
\newcommand{\sConstraint}[4]{\langle #1, #2, #3, #4\rangle}

\newcommand{\sConflict}[2]{(#1, #2)}

\newcommand\vExecLen{\ell}

\newcommand\vConfig{c}
\newcommand\vConfigS{s}
\newcommand\vConfigT{g}

\newcommand\vAgent{a}
\newcommand\vAgentb{b}
\newcommand{\vAgentsNumber}{k}

\newcommand\vVertex{v}

\newcommand\vTime{t}

\newcommand\vBool{\beta}

\newcommand\vConst{\textrm{constraints}}
\newcommand{\conjunctionconstraints}[2]{#1 \cup \{#2\}}
\newcommand\vPath{\pi}
\newcommand\vExec{\textrm{exec}}

\newcommand\vNode{\mathsf{n}}
\newcommand\vNodeChild{\mathsf{n}'}

\newcommand\vRoot{\textsf{root}}

\newcommand\vCost{\textrm{cost}}

\newcommand\aInsertRoot{\Call{insertRoot}{}}
\newcommand\aCreateChild{\Call{createChild}{}}
\newcommand{\aCSPname}{\Call{CSP}{}\xspace}
\newcommand\aCSP[3]{\Call{CSP}{#1, #2, #3}}
\newcommand\aSelf[3]{\Call{self}{#1, #2, #3}}
\newcommand\aOther[3]{\Call{other}{#1, #2, #3}}
\newcommand\aNeg[2]{\Call{neg}{#1, #2}}
\newcommand{\aOPEN}{\textsf{OPEN}\xspace}
\newcommand{\aCHILDREN}{\textsf{CHILDREN}\xspace}
\newcommand{\aBYPASS}{\textcolor{orange!50!black}{BYPASS}\xspace}

\makeatletter
\newcommand{\aLine}[1][.2pt]{\par\vskip.5\baselineskip\hrule height #1\par\vskip.5\baselineskip}
\makeatother

\newcommand\Raise{\textbf{raise}\xspace}

\newcommand\Discard{\textbf{discard}\xspace}

\newcolumntype{?}{!{\vrule width 1pt}}
\newcolumntype{:}{!{\vrule width 2pt}}

\tikzstyle{communication} = [dash pattern=on 1pt off 1pt, color=blue!50!white] 
\newcommand\Self{\textsc{Self}\xspace}
\newcommand\Other{\textsc{Other}\xspace}
\newcommand\Neg{\textsc{Neg}\xspace}
\newcommand\iSelf{\textsc{s}\xspace}
\newcommand\iOther{\textsc{o}\xspace}
\newcommand\iNeg{\textsc{n}\xspace}
\newcommand\CCBS{\normalfont\textrm{CCBS}\xspace}
\newcommand\CCBSSO{\CCBS_{\iSelf\iOther}\xspace}
\newcommand\CCBSN{\CCBS_{\iNeg}\xspace}
\newcommand\CCBSNSO{\CCBS\xspace}
\newcommand\CCBSS{\CCBS_{\iSelf}\xspace}

\newcommand\CCBSNSObold{\textbf{CCBS}\xspace}
\newcommand\CCBSSObold{\textbf{CCBS}_\textbf{SO}\xspace}
\newcommand\CCBSNbold{\textbf{CCBS}_\textbf{N}\xspace}

\begin{document}
\maketitle
\begin{abstract}
    We study a variant of the multi-agent path finding problem (MAPF) in which
    agents are required to remain connected to each other and to a designated
    base. This problem has applications in search and rescue missions where the
    entire execution must be monitored by a human operator. We re-visit the
    conflict-based search algorithm known for MAPF, and define a variant where
    conflicts arise from disconnections rather than collisions. We study
    optimizations, and give experimental results in which we compare our
    algorithms with the literature.
\end{abstract}

\section{Introduction}
\label{sec:intro}



In \emph{information-gathering}
missions, a group of robots are used to retrieve data at particular locations of an area (e.g. farm, building, etc).
An application is \emph{search \& rescue} missions which are often assisted by human operators,
and can be realized by unmanned aerial vehicles or other types of robots.
In some of these applications the robots must continuously remain connected, for instance,
in order to ensure a real-time video stream and to allow human operators to make quick decisions~\cite{abb-is17}.


In this paper, we consider the problem of
computing paths for a set of agents in which they remain connected at all
steps. We call this problem the \emph{Connected Multi-Agent Path Finding}
problem (CMAPF).
CMAPF was initially introduced and studied in~\cite{Hollinger:12:MCPC}. CMAPF
does not consider collisions. As advocated by Hollinger et al., a discretization
can take into account the geometry of the agents such that collisions can be
avoided by an on-board collision avoidance system. Also, with a small amount of
agents, collisions can be ignored by letting agents operate at different
altitudes (e.g. in drone applications). Thus, we suppose that several agents are
allowed to share the same position at a given time as done in
\cite{Hollinger:12:MCPC,Tateo:18:MCPP}.

Hollinger and Singh showed the NP-hardness of the problem given a bound on the
length of the execution, and provided an online algorithm. Tateo et al.,
in~\cite{Tateo:18:MCPP}, showed the PSPACE-completeness of the general decision
problem and gave two suboptimal sampling-based algorithms and an optimal
 DFS-based algorithm. In \cite{Charrier:19:RCPCA}, Charrier et al. show
the problem is in LOGSPACE when restricted to so-called sight-moveable graphs,
but the bounded version of the problem remains NP-hard.

 One may solve our problem offline with A*~\cite{Hart:68:ASTAR}. However, the
 group of agents is seen as one, and the algorithm is then exponential in the
 number of agents and the number of their available moves. That is an approach
 similar to the DFS-based algorithm of \cite{Tateo:18:MCPP}, which we
 compare to our algorithm.

In our work, we propose to take the approach of Conflict-Based Search (CBS)
algorithm introduced in~\cite{Sharon:12:CBS}. CBS solves another problem called
the Multi-Agent Path Finding (MAPF) problem, that focuses on
\emph{collision-free} paths. CBS starts by computing optimal paths for each
agent separately; when a collision between two agents is detected at a location,
the algorithm constrains one of the agents away from this location. CBS is,
intuitively, exponential in the number of conflicts.
 In general, CBS has significantly better performance in practice
than A*.

Our contribution is a complete and optimal algorithm called
\emph{Connectivity-Conflict-Based Search (CCBS)}. CCBS relies on ideas similar to
CBS (including the \emph{bypass} optimization \cite{Boyarski:15:CBSBP}),
but manipulates \emph{connectivity constraints} instead of collision constraints.

CMAPF is more challenging than MAPF. In MAPF, a collision conflict concerns only
 a pair of agents. That is why CBS solves directly a collision conflict in one step. In
 our case, a connectivity conflict concerns an arbitrary subset of agents. As
 solving a connectivity conflict is more demanding, we opt for strategies
 that do not solve directly a connectivity constraint, but guide the search
 towards a connected configuration.

 More precisely we consider three intuitive strategies when facing a disconnected configuration:
 
 \begin{itemize}
	\item \Neg: constrain one of the agent $\vAgent$ to \emph{not be} in its current location;
	\item \Self: constrain the disconnected agent $\vAgent$ to \emph{be} connected to
	some other agent or to the base;
	\item \Other: constrain another agent $\vAgentb$ to \emph{be} connected to the
	disconnected agent $\vAgent$.
\end{itemize}
 
As in CBS, strategy \Neg creates \emph{negative constraints}. 
As in done in \cite{Li:19:DS}, strategies
\Self and \Other create \emph{positive constraints}.

In our experiments, we consider the (optimal) variant $\CCBSN$ of $\CCBS$ that
applies only \Neg, and the (incomplete) variant $\CCBSSO$ that  applies only
\Self and \Other.  We compared our algorithm to A* with operator decomposition.,
Surprisingly, our experiments show that:
 \begin{itemize}
	\item $\CCBS$ clearly outperforms $\CCBSN$, which outperforms~A*;
	\item $\CCBSSO$, although incomplete, has the similar behavior as $\CCBS$ and outputs optimal plans in almost all cases.
 \end{itemize}
	
Finally, we discuss an optimization (called partial and selected splitting) that
saves $9\%$ of the memory consumption on average.

\paragraph{Outline}
We give the definition of the CMAPF problem, in Section~\ref{sec:preli}. Then,
in Section~\ref{sec:cbs}, we recall the CBS algorithm. We present our algorithm
CCBS, in Section~\ref{sec:algo}. In Section~\ref{sec:subopt}, we discuss completeness and optimality. Finally, in
Section~\ref{sec:exp}, we show our experimental results, and finish with
discussions, in Section~\ref{sec:concl}.

\section{Connected Multi-Agent Path Finding}
\label{sec:preli}

In this section, we formalize CMAPF \cite{Hollinger:12:MCPC,Tateo:18:MCPP,Charrier:19:RCPCA}.
%
%
The input is a
\emph{topological graph} specifying how agents can move (via movement edges) and
how they can communicate (via communication edges) with each other.

\label{sub:preli:graph}
\begin{definition}[Topological Graph] A \emph{topological graph} is a tuple
    $\sTopoGraph = \langle \sVertexSet, \sMoveRel, \sCommRel \rangle$, with
    $\sVertexSet$ a finite set of vertices containing a distinguished element
    $\sBase$ called the \emph{base}, $\sMoveRel~\subseteq \sVertexSet\times
    \sVertexSet$ a set of undirected movement edges and $\sCommRel \subseteq
    \sVertexSet \times \sVertexSet$ a set of undirected communication edges.
\end{definition}

In this work, we restrict to graphs in which all vertices contain a movement
self-loop. This means that agents can always \emph{idle} at any vertex.
Figure~\ref{figure:example-cmapf} gives an example of a topological graph with 7
vertices.

\label{sub:preli:exec}

\begin{definition}[Execution] \label{def:execution}
	An \emph{execution} $\vExec$ of length $\vExecLen$ with $\vAgentsNumber$
	agents in a topological graph $\sTopoGraph$ is a collection of  $\vAgentsNumber$ paths of
	length $\vExecLen$, one for each agent.
\end{definition}

As we are interested in the makespan (that is the maximum of the lengths), in
Definition~\ref{def:execution}, we suppose w.l.o.g. that paths are of the same
lengths; if not, simply consider that agents can idle at their destinations.
Agents are numbered from~$1$ to~$\vAgentsNumber$. Given an execution $\vExec$,
$\vExec_\vAgent$ is the path of agent $\vAgent$. We denote by
$\vExec_\vAgent[t]$, the vertex occupied by agent~$\vAgent$ at the $t$-th step.
We denote by $\vExec[\vTime]$, the positions of all the agents at time $t$, i.e.
$\vExec[\vTime]$ is the vector $(\vExec_1[\vTime], \dots,
\vExec_\vAgentsNumber[\vTime]) \in \sVertexSet^\vAgentsNumber$. Such a vector
$\vConfig \in \sVertexSet^\vAgentsNumber$ is called a \emph{configuration}.

\begin{definition}
	A configuration $\vConfig\in \sVertexSet^\vAgentsNumber$ is \emph{connected}
	if the vertices $\sBase, \vConfig_1, \dots, \vConfig_\vAgentsNumber$ form a
	connected subgraph w.r.t communication edges ($\sCommRel$). Otherwise, we
	say that $\vConfig$ is disconnected. An execution $\vExec$ of
	length~$\vExecLen$ is said to be connected if $\vExec[\vTime]$ is connected,
	for all  $1\leq \vTime\leq \vExecLen$.
	\label{definition:connected}
\end{definition}

Definition~\ref{definition:connected} captures agents that are connected to the
base via multi-hop (that is, an agent $\vAgent$ is connected to the base if
there is a sequence of agents connecting $\vAgent$ to the base). An example of
an execution using multi-hop connection is depicted in Example~\ref{ex:sm-exec}.

We define the following optimization problem called \emph{connected multi-agent
path finding problem} (CMAPF), in which we require the group of agents to be
connected via communication edges with the base during the entire execution. We
minimize the makespan of the execution.

\begin{definition}[CMAPF] Given a topological graph $\sTopoGraph = \langle
	\sVertexSet, \sMoveRel, \sCommRel \rangle$, number $\vAgentsNumber$ of
	agents, two configurations $\vConfigS, \vConfigT \in V^\vAgentsNumber$, find
	a connected execution $\vExec$ of minimum length $\vExecLen$
	with~$\vAgentsNumber$ agents in $\sTopoGraph$ such that $\vExec[1] =
	\vConfigS$ and $\vExec[\vExecLen] = \vConfigT$.
	\label{def:pb:reach}
\end{definition}

\begin{example}
	Consider the topological graph of Figure~\ref{figure:example-cmapf}.
	Consider the instance of CMAPF with $\vAgentsNumber=2$ agents with starting
	configuration~$s=(v_1,v_4)$ and goal configuration~$g=(v_3,v_6)$. The
	execution $\set{(v_1, v_2, v_3), (v_4,v_5, v_6)}$ is not connected. Indeed,
	at the second step the configuration is $(v_2,v_5)$: the first agent is
	disconnected, that is the set of vertices $\{v_2,v_5,B\}$ do not form a
	connected graph with the communication relation. However, the execution
	$\set{(v_1, v_2, v_3, v_3), (v_4, v_4, v_5, v_6)}$ is connected: the sets
	$\set{v_1, v_4, B}$, $\set{v_2, v_4, B}$, $\set{v_3, v_5, B}$ and $\set{v_3,
	v_6, B}$ all form connected graphs with the communication relation.
	Actually, that execution is an optimal solution of this CMAPF instance.
	\label{ex:sm-exec}
\end{example}

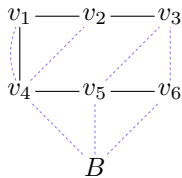
\begin{figure}
	\begin{center}
		\newcommand{\communicationsight}[2]{\draw[communication] (#1) edge[bend left=20] (#2);}
		\newcommand{\commLine}[2]{\draw[communication] (#1) edge (#2);}
		\begin{tikzpicture}[node distance=1cm,inner sep=1pt,minimum size=1pt]
		\node (1) {$v_1$};
		\node[right of=1] (2) {$v_2$};
		\node[right of=2] (3) {$v_3$};
		\node[below of=1] (4) {$v_4$};
		\node[right of=4]  (5) {$v_5$};
		\node[right of=5] (6) {$v_6$};
		\node[below of=5] (B) {$B$};

		\draw[-] (1) -- (2) -- (3) (1) -- (4) -- (5) -- (6);
		\commLine B 4
		\commLine B 5
		\commLine B 6
		\communicationsight 4 1
		\commLine 4 2
		\commLine 6 3
		\commLine 5 3
		\end{tikzpicture}
	\end{center}
  \caption{An example of a topological graph where plain edges are movement edges,
  and dotted ones are communication edges.
  \label{figure:example-cmapf}}
\end{figure}


\section{Conflict-Based Search}
\label{sec:cbs}
We recall the conflict-based search (CBS) algorithm \cite{Sharon:12:CBS} that solves the MAPF problem under collision constraints
but without connectivity.
CBS is composed of two levels: the \emph{high-level}
builds a \emph{constraint tree} while the \emph{low-level} finds
optimal single-agent paths.




The \emph{constraint tree} $\sConstTree$ is
composed of nodes $\vNode$ with the following attributes:

\begin{itemize}
\item $\vNode.\vConst$ - A finite set of constraints;
\item $\vNode.\vExec$ - An execution as defined in Definition~\ref{def:execution};
\item $\vNode.\vCost$ - The current cost of the execution (i.e. its length).
\end{itemize}
Initially, the root node contains no constraints, and its execution consists of a set of shortest paths computed independently for each agent. The cost is the 
maximum of the lengths of these paths.
Here, we 
extend all paths to have the same length by letting agents idle at their goal vertices.

A \emph{conflict} of an execution is a time point where a pair of agents are in
collision. If a constraint tree node without conflicts is created, then the algorithm returns
the execution. Otherwise, it chooses a conflict; say
agents~$a,a'$ collide at time~$t$. The algorithm creates two successor nodes
obtained by adding the negative constraints requiring, respectively, that
agent~$a$ must not be at~$\vExec_a[t]$ at time~$t$, and agent~$a'$ must not be
at $\vExec_{a'}[t]$ at time~$t$. Since all solutions must satisfy one of these
constraints, if the optimal solution is compatible with the current node, it
must be compatible with one of these successors.

For each successor node, one updates the shortest path for the agent with the
new constraint. This computation consists in the \emph{lower level} of CBS, and
is typically done using time-space A*~\cite{Silver:05:CP}.






\paragraph{Bypassing Conflict}

The bypassing conflict optimization \cite{Boyarski:15:CBSBP} works as follows. If
the execution~$\vNodeChild.\vExec$ in a successor $\vNodeChild$ of $\vNode$ has
the same cost but has fewer conflicts than the execution~$\vNode.\vExec$ of
$\vNode$, then one replaces~$\vNode.\vExec$ by~$\vNodeChild.\vExec$ and deletes
all children of $\vNode$.


\section{Connectivity-Conflict-Based Search}
\label{sec:algo}

In this section, we describe our algorithm called Connectivity-Conflict-Based
Search (CCBS) for the CMAPF problem. Our algorithm is adapted from conflict-based
search. 

\subsection{Positive and Negative Constraints}

Connectivity constraints are more demanding than collision constraints
found in CBS. Connectivity is a property involving all
agents, rather than being a local property involving only a pair of agents.

That is why, we do not try to solve a connectivity constraint in one single
step. Instead, we help the search to reach connected configurations. It seems
natural that solving disconnection requires to enforce some agent to be at a
connected location. In addition to \emph{negative} constraints,
we also use \emph{positive} ones that \emph{require} an agent to be at a
position, typically connected to another agent.

Thus, our constraints are of the form $\sConstraint \vAgent \vVertex \vTime
\vBool$, where $\vAgent$ is the constrained agent, $\vVertex$ the vertex on
which it is constrained, $\vTime$ the time-step at which the constraint applies
and $\vBool\in\{\top,\bot\}$ is a Boolean value that specifies whether the
constraint is positive or negative. The constraint $\sConstraint \vAgent
\vVertex \vTime \top$ (resp. $\sConstraint \vAgent \vVertex \vTime \bot $) means
that the agent $\vAgent$ should be (resp. should not be) at vertex $\vVertex$ at
time~$\vTime$.

Positive constraints were recently used in CBS~\cite{Li:19:DS} as an
optimization to create disjoint splits of a node. We here apply them in a
different context, but we use their low-level algorithm to compute
constrained shortest paths.

\subsection{The Low-Level: Constrained Shortest Paths}
\label{sub:algo:csp}
We use the algorithm described in \cite{Li:19:DS} to compute the constrained
shortest paths for individual agents.

Given a set of positive and negative constraints, a start and a goal vertex,
we use the positive constraints as timely ordered \emph{landmarks}. We compute
the path from the start location to the first landmark, then from the first
landmark to the second and so on up to the goal. This is done iteratively with
the original low-level of CBS using the negative
constraints~\cite{Sharon:12:CBS}. No time bound is put on the path to the goal, while
one is used on landmarks, which correspond to the time given
in the positive constraint.
 Remark that if a landmark cannot be reached in the
given time then there is no path satisfying the constraints.

\subsection{The High-Level: The Conflict Tree}
\label{sub:algo:ct}
In the original CBS, a conflict involves two agents, thus CBS creates two
successor nodes to solve it. In our case, a \emph{conflict} in an execution is a
disconnected configuration; thus, conflicts involve \emph{all} agents. So we
create a larger number of successors reflecting the many ways to solve the
conflict.

Let us explain the strategies used to generate constraints to handle
connectivity. First, as in CBS, we add negative constraints but for all agents.
More precisely, for a given disconnected configuration~$c$ at time~$t$, we create
a successor node for each agent~$\vAgent$, with the negative constraint that
forbids agent~$\vAgent$ to be at $c_{\vAgent}$ at time~$t$.

This strategy,
called \Neg, makes the search complete but, as we will see, it is
inefficient alone to solve CCBS.
Indeed, when CBS encounters a collision between two agents, in MAPF, the usage of
a negative constraint is enough to solve this collision. However, the
use of a negative constraints, in CMAPF, does not guarantee that the
disconnection is solved immediately. The idea is rather to guide the search towards
connected configurations by letting agents stay together.

To obtain an efficient algorithm, we will add positive constraints, according to
two strategies called $\Self$ and $\Other$. For a given conflict, we consider a disconnected
agent $\vAgent$. The strategy \Self
constrains agent $\vAgent$ to be at a position connected to some
other agent or to the base. The strategy \Other forces an arbitrary agent to a
position where it becomes connected to agent $\vAgent$. The \Self
and \Other strategies can be seen as shortcuts to negative constraints.

\begin{algorithm}
	\algorithmfont
	\begin{algorithmic}[1]
        \caption{High-Level of CCBS} \label{alg:CBS}
		\Require A topological graph $\sTopoGraph=\langle\sVertexSet, \sMoveRel,
		\sCommRel\rangle$, an initial configuration $\vConfigS$ and a goal
		configuration $\vConfigT$ (considered as global variables)
		\aLine
        \State \aInsertRoot
        \While{\aOPEN is not empty}
            \State $\vNode \gets$ best node from \aOPEN
            \If{$\vNode$ has no conflict}
                \State \Return $\vNode.\vExec$
            \EndIf

			\State CHILDREN := empty list

			 \State {$\sConflict \vTime \vAgent \gets$ \raisebox{-1.5mm}{\begin{minipage}{6cm}a time-step and a disconnected agent in a 
				conflict in $\vNode.\vExec$\end{minipage}}}

            \State \label{alg:line:self} ${
                \aSelf{\vNode}{\vTime}{\vAgent}
            }$

            \State \label{alg:line:other} ${
                \aOther{\vNode}{\vTime}{\vAgent}
            }$

            \State ${
            	\aNeg{\vNode}{\vTime}
            }$\label{alg:line:neg}

            \If{\aBYPASS was raised}
            	\State \Discard \aCHILDREN

            \Else
            \State 	insert all nodes in \aCHILDREN into \aOPEN

            \EndIf
			\label{alg:line:bend}
        \EndWhile
	\end{algorithmic}
\end{algorithm}

The overall algorithm is given in Algorithm~\ref{alg:CBS}, which shows how the
constraint tree is created. It maintains a priority queue called \aOPEN which
stores the set of leaf nodes that have not been expanded yet. It runs as long as
a solution has not been found and this queue is non-empty. At each iteration, a
best node is picked (Line~3). If that node has no conflict, it means that a
solution is found (Line~5). If the execution of the node contains conflicts,
then a conflict is chosen (arbitrarily) and child nodes are created with the
strategies (from Line~8 to Line~10). The last part of the algorithm shows the
bypass optimization.

The \aInsertRoot\ procedure in Algorithm~\ref{alg:procedures} shows the
initialization step where the root of the constraint tree is created. Here,
\aCSPname refers to the Constrained Shortest Path computation of the low-level
described in Subsection~\ref{sub:algo:csp}. This call returns a shortest path
satisfying the given set of constraints between the starting node
$\vConfigS_\vAgent$ and the target node $\vConfigT_\vAgent$ of agent $\vAgent$.
The procedures \Self, \Other and \Neg describe the creation of conflicts of
respective types. Furthermore, each one of these procedures call \aCreateChild,
which is responsible for creating child nodes and detecting the BYPASS condition
(for which the bypass optimization is applied).

\begin{algorithm}[ht]
	\algorithmfont
	\begin{algorithmic}[1]
        \caption{Sub-procedures}
		\Procedure{insertRoot}{}
		\State $\vRoot \gets$ new node
		\State $\vRoot.\vConst \gets \emptyset$

		\State \textbf{for }{all agents $\vAgent$} \textbf{do }
		$\vRoot.\vExec_\vAgent \gets
		\aCSP{\vConfigS_\vAgent}{\vConfigT_\vAgent}{\emptyset}$

		\State $\vRoot.\vCost \gets \max \set{|\vRoot.\vExec_\vAgent| \text{where
		$\vAgent$ is an agent}} $

		\State insert $\vRoot$ to \aOPEN
		\EndProcedure
		\aLine
		\Procedure{createChild}{node~$\vNode$, $\langle
		\vAgent, \vVertex, \vTime, \vBool\rangle$}
		\State $\vNodeChild \gets$ new node

		\State $\vNodeChild.\vConst \gets
		\conjunctionconstraints{\vNode.\vConst}{\langle \vAgent, \vVertex,
		\vTime, \vBool\rangle}$

		\State \textbf{for }{all agents $\vAgentb$} \textbf{do }
		$\vNodeChild.\vExec_\vAgentb \gets \vNode.\vExec_\vAgentb$

		\State $\vNodeChild.\vExec_\vAgent \gets
		\aCSP{\vConfigS_\vAgent}{\vConfigT_\vAgent}{\vNodeChild.\vConst}$

		\State $\vNodeChild.\vCost \gets \max \set{|\vNodeChild.\vExec_\vAgent|,
		\text{$\vAgent$ an agent}} $

		\If{$\vNodeChild.\vCost = \vNode.cost$ and $\vNodeChild$ has less
		conflicts than $\vNode$}

		\State $\vNode.\vExec_\vAgent \gets \vNodeChild.\vExec_\vAgent$

		\State \Raise \aBYPASS
		\EndIf

		\State insert $\vNodeChild$ to \aCHILDREN
		\EndProcedure
		\aLine

		\Procedure{self}{node $\vNode$, time $t$, agent $\vAgent$}
		\ForAll{agents $\vAgentb$ different from $\vAgent$} \ForAll{vertices
		$\vVertex''$ s.t.
		$\insCommRel{\vNode.\vExec_{\vAgentb}[\vTime]}{\vVertex''}$}
		\If{$\vVertex'' \neq \vExec_{\vAgent}[\vTime]$} \State
		$\aCreateChild(\vNode, \sConstraint\vAgent{\vVertex''}\vTime\top)$
		\EndIf \EndFor \EndFor \EndProcedure

		\ForAll{vertices $\vVertex''$ s.t. $ \insCommRel {\sBase}{\vVertex''}$}
		\If{$\vVertex'' \neq \vExec_{\vAgent}[\vTime]$}
		\State $\aCreateChild(\vNode, \sConstraint\vAgent{\vVertex''}\vTime\top)$
		\EndIf\EndFor

		\aLine

		\Procedure{other}{node $\vNode$, time $t$, agent $\vAgent$}
        \ForAll{agents $\vAgentb$ different from $\vAgent$} \ForAll{vertices
        $\vVertex'$ s.t. $\insCommRel{\vNode.\vExec_\vAgent[\vTime]}{\vVertex'}$}
        \If{$\vVertex' \neq \vExec_{\vAgentb}[\vTime]$} \State
        $\aCreateChild(\vNode, \sConstraint{\vAgentb}{\vVertex'}\vTime\top)$
        \EndIf \EndFor \EndFor \EndProcedure \aLine \Procedure{neg}{node
        $\vNode$, time $t$, agent $\vAgent$} \ForAll{agents $\vAgent$} \State
        $\aCreateChild(\vNode,
        \sConstraint{\vAgent}{\vNode.\vExec_\vAgent[\vTime]}\vTime\bot)$ \EndFor
        \EndProcedure
		\label{alg:procedures}
	\end{algorithmic}
\end{algorithm}

While the strategy \Neg creates~$\vAgentsNumber$ child nodes, the number of
child nodes created by \Self and \Other is rather quadratic, in
$O(\vAgentsNumber \times |\sVertexSet|)$  in each case, where $\vAgentsNumber$ is the number of
agents and $|\sVertexSet|$ is the number of vertices. Thus, the
constraint tree might grow too quickly. We will show that despite this high branching
factor, we do obtain satisfactory results in our benchmarks.
See also discussion in Section~\ref{sec:exp}.



\section{Discussion on Completeness and Optimality}
\label{sec:subopt}
\label{sec:subalgo}
In this subsection, we consider the variant of $\CCBS$ called $\CCBSN$ in which only strategy $\Neg$ is applied, but not 
$\Self$ and $\Other$ (we remove lines 8 and 9 in Algorithm~\ref{alg:CBS}.
We also consider the variant $\CCBSSO$ in which only strategies $\Self$ and $\Other$ are applied, but not strategy $\Neg$ (line 10 is deleted).
As long as we apply the strategy $\Neg$, we obtain a complete and optimal
algorithm. This can be proved similarly as in~\cite{Sharon:12:CBS}:

\begin{theorem}
	Both $\CCBSN$ and $\CCBSNSO$ are complete and optimal.
\end{theorem}

\subsection{Incompleteness of $\CCBSSO$}

However, strategies $\Self$ and $\Other$ alone lead to an incomplete algorithm.
%
 We will mainly study the variant $\CCBSSO$ obtained by omitting the strategy
 $\Neg$. 
 We observe that the remaining positive constraints guide the search very
 quickly towards a solution in our experiments.

 We distinguish a class of topological graphs, called \emph{sight-moveable
 graphs}. In fact, in a typical radius discretization, if the agents restrict
 their communication to links that do not cross obstacles then the obtained
 topological graph is sight-moveable. While being a strong restriction, such
 communication may guarantee a more reliable connection. Furthermore, the agents
 can still exploit links crossing obstacles while not being required during the
 execution. The formal definition is given below.

 Our original motivation was to develop an efficient algorithm for this class. In
 fact, the theoretical complexity of deciding the existence of a connected plan
 was shown to be in LOGSPACE for sight-moveable graphs \cite{Charrier:19:RCPCA},
 while it is PSPACE-complete for general graphs~\cite{Tateo:18:MCPP}. We will
 first introduce this class and then formally study the properties of $\CCBSSO$.

 \paragraph{Sight-Moveable Graphs}

 We recall the class of \emph{sight-moveable} topological graphs introduced
 in~\cite{Charrier:19:RCPCA}. Whenever an agent can communicate with another
 node, then it can also move to that node while maintaining the communication. A
 sight-moveable graph can be obtained on a discretized graph by restricting
 communication in line of sight. 

 Formally, a \emph{sight-moveable topological graph}  has an undirected movement
 relation such that for all vertices $\vVertex, \vVertex' \in \sVertexSet$, for
 all $\vVertex \sCommRel \vVertex'$, there is a sequence of vertices
 $\vPath=\langle \vPath_0, \vPath_1, \dots, \vPath_\ell\rangle$ of size $\ell$ such
 that $\vPath_0 = \vVertex$, $\vPath_\ell = \vVertex'$  $\vPath_i \sMoveRel \vPath_{i+1}$ and $\vPath_i \sCommRel v'$, for all $0 \leq i < \ell$.

 \begin{example}
   Figure~\ref{figure:sight-moveabletopologicalgraph} shows a sight-moveable
   graph. For instance, $B \coms q_5$ and, indeed, there is a path $B \sMoveRel
   q_4 \sMoveRel q_5$ with $B \coms q_4$. In contrast, the graph in
   Figure~\ref{figure:example-cmapf} is not sight-moveable, since $v_4\coms v_2$
   but it is not possible to go from~$v_4$ to~$v_2$ while maintaining the
   communication with~$v_2$.
 \end{example}

 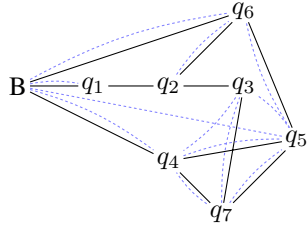
\begin{figure}
 	\begin{center}
 		\newcommand{\communicationsight}[2]{\draw[communication] (#1) edge[bend left=10] (#2);}
 		\begin{tikzpicture}[node distance=1cm,inner sep=1pt,minimum size=1pt]
 		\node (B) {B};
 		\node[right of=B] (1) {$q_1$};
     \node[right of=1] (2) {$q_2$};
     \node[right of=2] (3) {$q_3$};
     \node[below of=2] (4) {$q_4$};
     \node[below right of=4] (7) {$q_7$};
     \node[below right of=3] (5) {$q_5$};
     \node[above of=3] (6) {$q_6$};
     \draw[-] (B) -- (1) -- (2) -- (3) (B) -- (4) -- (5) -- (6) -- (B)
       (4) -- (7) -- (5) (7) -- (3)
       (2) -- (6);
     \communicationsight B 1
     \communicationsight B 4
     \draw[communication] (B) -- (5);
     \communicationsight B 6
     \communicationsight 2 6
     \communicationsight 3 5
     \communicationsight 3 4
     \communicationsight 5 6
     \communicationsight 4 5
     \communicationsight 7 4
     \communicationsight 7 3
     \communicationsight 7 5

 		\end{tikzpicture}
 	\end{center}
   \caption{A sight-moveable topological graph where plain edges represent
     movement edges, and the dotted ones represent communication edges.
   \label{figure:sight-moveabletopologicalgraph}}
 \end{figure}

\begin{theorem}
  $\CCBSSO$ is not complete.
  \label{thoerem:SelfOtherInComplete}
\end{theorem}

\begin{proof}
	Consider the sight-moveable topological graph of Figure~\ref{figure:sight-moveabletopologicalgraph}.
	Consider the instance with two agents: $s=(q_4,q_3)$ and~$g=(B,B)$.
	
	This instance has a solution, for instance, the connected execution from configuration $s$ to $g$ made up of the following paths:
	$(q_4,q_5,q_6,B,B)$ for agent $a_1$,  and 
	$(q_3,q_3,q_2,q_1,B)$ for agent $a_2$.

	However, let us show that $\CCBSSO$ does not find a solution.
	The first shortest paths generated by the algorithm are, respectively,
	$(q_4, B)$ and~$(q_3, q_2, q_1, B)$. Thus agent~$a_2$ is disconnected at the second step.
	The constraints added by the \Self strategies are not satisfiable.
	In fact, agent $\vAgent_2$ can only move to $q_2$ and~$q_7$ at step 2, and none
	of them are connected to~$B$ (where agent 1 is at step 2.)
	Furthermore, there is only one constraint of type \Other added by the algorithm,
	and it consists at placing agent $\vAgent_1$ at $q_6$ at step~$2$, so as to 
	connect it to agent~$2$ at $q_2$. However, this is not satisfiable either
	since agent~$\vAgent_1$ cannot move there in one step.
	So the search is stuck and no solution
	is returned.
	
\end{proof}

Interestingly, even if $\CCBSSO$ is not complete in theory, our experiments
promote the use of $\Self$ and $\Other$ alone: $\CCBSSO$ is slightly faster
while outputting optimal plans in almost all cases.

\subsection{Completeness of $\CCBSS$ when all agents start at the base}
On the bright side, we show that the algorithm is complete on sight-moveable graphs if all agents start
at the base~$B$. This is an interesting case since in some applications, agents
(say, drones) are all launched from the base. Thus, on a typical mission,
reaching a given configuration from the base would be the initial task. The
following lemma shows that only the $\Self$ constraints are required to ensure
completeness. The variant~$\CCBSSO$ is complete a fortiori.

\begin{theorem}
	\label{thoerem:SelfComplete}
	$\CCBSS$ is complete on sight-moveable graphs when all agents start at the
	base.
\end{theorem}

\begin{proof}
	Consider a sight-moveable topological graph~$G$, and a connected
	configuration~$g = (g_1,\ldots,g_\vAgentsNumber)$. 
	We are going to construct a
	particular execution from~$B^\vAgentsNumber$ to~$\vConfigT$, in the same way
	as in Prop. 19 in~\cite{Charrier:19:RCPCA}. Let us order the
	nodes~$B,\vConfigT_1,\ldots,\vConfigT_\vAgentsNumber$ into
	$\vConfigT_{i_0},\vConfigT_{i_1},\vConfigT_{i_2},\ldots,
	\vConfigT_{i_\vAgentsNumber}$ with~$\vConfigT_{i_0} = B$, such that for
	all~$1\leq j\leq \vAgentsNumber$, $\vConfigT_{i_j}$ is connected to
	some~$\vConfigT_{i_k}$ with $0 \leq k < j$. This order induces a tree, and can
	be obtained by breath-first search in a $\coms$-spanning tree of $V'$ at root
	$B$; see Figure~\ref{figure:sm} for an example. Let us fix such a spanning
	tree~$T_{\vConfigT}$.
	
	\begin{figure}
		\begin{center}
			\begin{tikzpicture}[yscale=1]
			\node (B) {$B = \vConfigT_{i_0}$};
			\node (ci1) at (1.5, 0) {$\vConfigT_{i_1}$};
			\node (ci2) at (2.75, 0) {$\vConfigT_{i_2}$};
			\node (ci3) at (4, 0) {$\vConfigT_{i_3}$};
			\node (ci4) at (4.5, -0.4) {$\vConfigT_{i_4}$};
			\node (ci5) at (5.5, -0.4) {$\vConfigT_{i_5}$};
			\node (ci6) at (6.5, -0.4) {$\vConfigT_{i_6}$};
			\draw[communication] (B) -- (ci1);
			\draw[communication] (B) edge[bend left=15] (ci2);
			\draw[communication] (ci1) edge[bend right=15] (ci3);
			\draw[communication] (B) edge[bend right=15] (ci4);
			\draw[communication] (ci4) -- (ci5);
			\draw[communication] (ci5) -- (ci6);
			\end{tikzpicture}  	\end{center}
		\caption{Example of an ordering of nodes in $V' = \set{\vConfigT_1, \dots, \vConfigT_n, B}$.}
		\label{figure:sm}
	\end{figure}
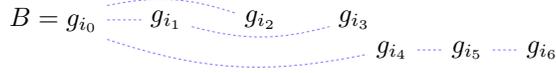
	
	By the sight-moveable property, for any pair~$\vConfigT_i,\vConfigT_j$
	connected in this tree, there is a path $\rho_{\vConfigT_i,\vConfigT_j}$ along
	which an agent can move from~$\vConfigT_i$ to~$\vConfigT_j$ while staying
	connected with~$\vConfigT_i$ along the way. Informally, the executions we
	define have the following property. Any agent who arrives at their destination
	stays there.   Furthermore, for any agent~$a$, if~$a'$ denotes the agent such
	that $\vConfigT_{a'}$ is the parent of~$\vConfigT_a$ in $T_{\vConfigT}$,
	then, agent~$a$ strictly shares the same position as~$a'$ until $a'$ reaches
	their target. Agents move in groups following
	paths~$\rho_{\vConfigT_i,\vConfigT_j}$ to move from one target configuration
	to another. For instance, in Fig.~\ref{figure:sm}, the executions for
	agents~$i_4,i_5,i_6$ consist in them moving all together to~$\vConfigT_{i_4}$
	(via $\rho_{\vConfigT_{i_0},\vConfigT_{i_4}}$) then $i_5,i_6$ moving together
	to~$\vConfigT_{i_5}$ (via $\rho_{\vConfigT_{i_4},\vConfigT_{i_5}}$), and~$i_6$
	to moving alone to~$\vConfigT_{i_6}$ (via
	$\rho_{\vConfigT_{i_5},\vConfigT_{i_6}}$).  The sight-moveable (SM) property ensures that the
	agents are always connected along this execution. We are going to define a set
	of executions based on  this particular execution.
	
	Let~$\rho_i$ denote the path $\rho_{\vConfigT_{\alpha_0},
		\vConfigT_{\alpha_1}} \rho_{\vConfigT_{\alpha_1}, \vConfigT_{\alpha_2}} \ldots
	\rho_{\vConfigT_{\alpha_{k-1}}, \vConfigT_{\alpha_k}}$ with~$\alpha_0 = 0,
	\alpha_k=i$, and $\vConfigT_{\alpha_1},\ldots,\vConfigT_{\alpha_k}$ is the
	path from the root to node~$\vConfigT_i$ in $T_{\vConfigT}$. Intuitively, this
	is the path that is taken by agent~$i$.
	
	Let~$\vExec$ denote the execution constructed above. Observe that all
	configurations of $\vExec$ are connected by the SM property. In fact, all
	nodes visited along $\rho_{\vConfigT_{\alpha_0,\alpha_1}}$ are directly
	connected to the base since~$\alpha_0=0$ and~$\vConfigT_0$ is the base. By
	construction, an agent moves along the
	path~$\rho_{\vConfigT_{\alpha_i,\alpha_{i+1}}}$ only when the agent whose
	target is $\vConfigT_{\alpha_i}$ has reached this node and remains there. By the
	SM property, all nodes of this path are connected with~$\vConfigT_{\alpha_i}$.
	This argument, applied inductively, shows that all reached configurations are
	connected~(see \cite{Charrier:19:RCPCA}).
	
	We are going to show that in the CT created by $\CCBSS$, there is a branch~$b$
	such that along all its nodes~$\vNode$, either there is a solution or the
	following invariant holds:
	\begin{itemize}
		\item $\vExec$ satisfies $\vNode.\vConst$,
		\item for all agents~$\vAgent$, $\forall t \geq |\vExec_\vAgent|, \vNode.\vExec_\vAgent[t] = g_\vAgent$.
	\end{itemize}
	That is, $\vExec$ is compatible with current constraints, moreover, once an
	agent reaches their goal node, they remain there.
	
	The invariant holds at the initial node since there are no constraints, and
	since the algorithm produces a shortest path for each agent separately, which
	are shorter or equal in length than the paths prescribed by~$\vExec$.
	
	Assume that the property holds at some node $\vNode$. If there is no conflict,
	then the algorithm has found a solution. Otherwise, consider any conflict picked
	by the algorithm, say, $(\vTime,\vAgent)$ (i.e. agent~$\vAgent$ is disconnected
	at step~$\vTime$). We argue that $\aSelf{\vNode}{t}{a}$ creates a child node
	with constraint $\sConstraint \vAgent v \vTime \top$ where~$v$ is the
	$\vTime$-th vertex of~$\rho_\vAgent$.
	For each~$1\leq k \leq \vAgentsNumber$,
	let~$\vConfigT_{\alpha_0},\vConfigT_{\alpha_1},\ldots,\vConfigT_{\alpha_k}$ be
	the path from the root to $\alpha_k$ in~$T_{g}$. If~$t \leq
	|\rho_{\vConfigT_{\alpha_0},\vConfigT_{\alpha_1}}|$, then $v$ is connected
	to the base, so the above constraint will be added. Notice that~$\vExec$ is
	compatible with the newly added constraint; furthermore, when the algorithm
	recomputes a shortest path for agent~$\vAgent$, it must find one whose length is
	not more than that of~$\vExec_\vAgent$ since the latter is a candidate path
	satisfying the constraints. The invariant thus holds in the child node.
	
	Otherwise, let~$j$ denote the largest index such that
	\[
	|\rho_{\vConfigT_{\alpha_0},\vConfigT_{\alpha_1}}\rho_{\vConfigT_{\alpha_1},\vConfigT_{\alpha_2}} \ldots \rho_{\vConfigT_{\alpha_j-1},\vConfigT_{\alpha_j}}|
	\leq t.
	\]
	By the invariant, agent~$\alpha_j$ is at vertex~$\vConfigT_{\alpha_j}$ at
	time~$t$. But~$\vConfigT_{\alpha_j}$ is connected to~$v$ since~$v$ is on the
	path $\rho_{\vConfigT_{\alpha_j},\vConfigT_{\alpha_{j+1}}}$. Thus, some child
	node of~$\vNode$ will be created with the above constraint. The
	execution~$\vExec$ satisfies the newly added constraint by definition.
	Furthermore, $\vExec_\vAgent$ is a path from the base to~$\vConfigT_{\vAgent}$ so the
	newly computed shortest path for~$\vAgent$ must be of length at
	most~$|\vExec_\vAgent|$, which shows the second part of the invariant.
\end{proof}


\section{Experimental Results}
\label{sec:exp}
In this section, we evaluate optimal algorithms $\CCBSNSO$, $\CCBSN$, and $\CCBSSO$ that is incomplete.

\subsection{Benchmarks}

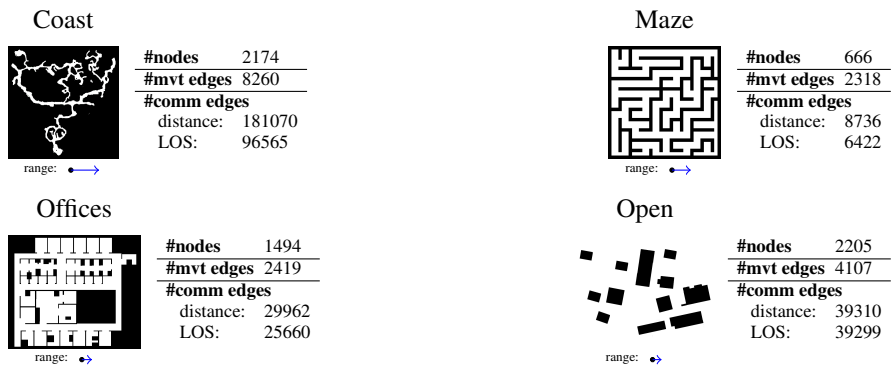
\begin{figure*}[t]
    \newcommand{\mapheight}{15mm}
    \newcommand{\subfiguremapheight}{.11}
    \newcommand{\ypositionradius}{-0.9}
    \newcommand{\benchmark}[7]{
        \begin{tabular}{c}
            #1 \\
            \begin{tikzpicture}
                \node {	\includegraphics[height=\mapheight]{figures/#2.png} };
                \node at (-0.3, \ypositionradius) {\tiny range: };
                \node[draw, fill=black, circle, inner sep=0.2mm] at (0.1, \ypositionradius) {};
                \draw[->, blue] (0.1, \ypositionradius) -- (0.1+#7*1.5, \ypositionradius);
            \end{tikzpicture}
        \end{tabular}\hspace{-2mm}
        \raisebox{0cm}{
            \scriptsize
            \begin{tabular}{ll}
                \textbf{\#nodes} & #3 \\\hline
                \textbf{\#mvt edges} \hspace{-2mm} & #4 \\\hline
                \multicolumn{2}{l}{\textbf{\#comm edges}} \\
                ~~ distance: & #5 \\
                ~~ LOS: & #6 \\
            \end{tabular}
        }
	}
    \benchmark{Coast}{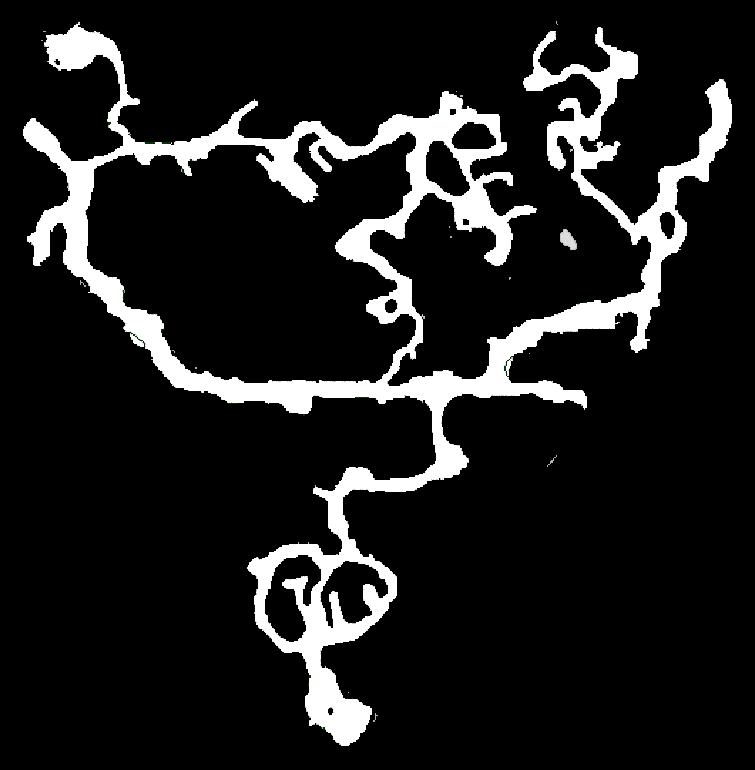}{2174}{8260}{181070}{96565}{0.25}\hfill
    \benchmark{Maze}{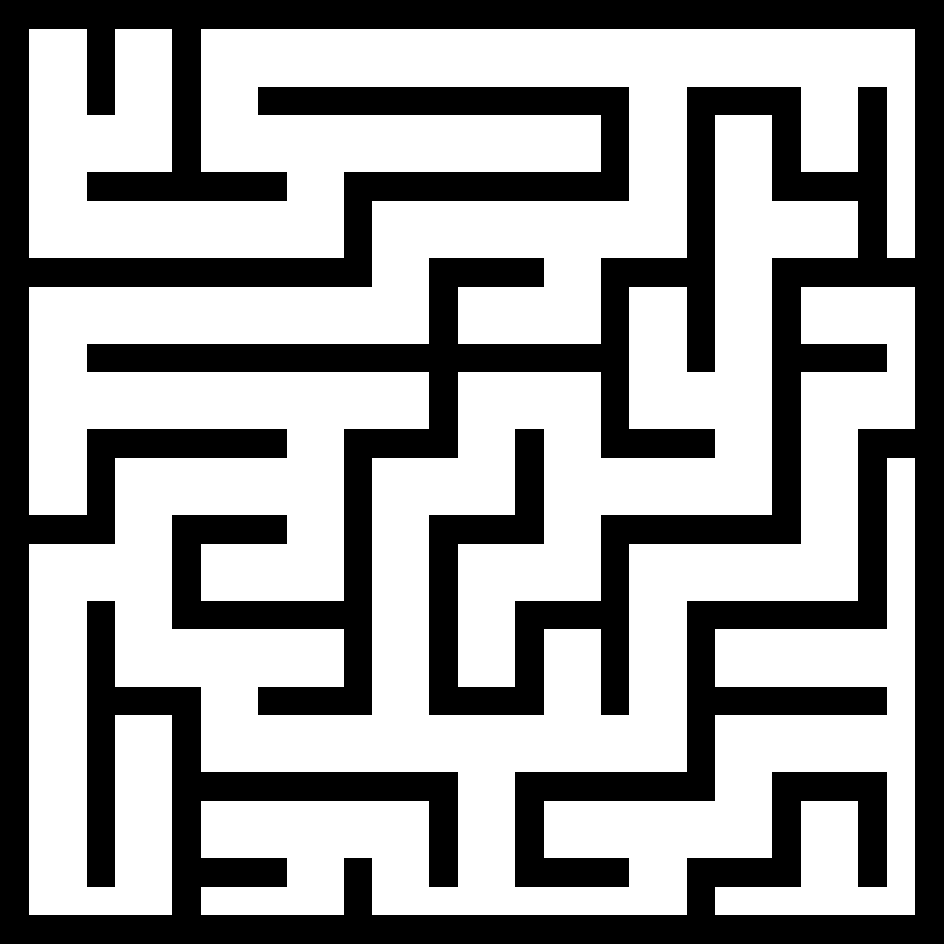}{666}{2318}{8736}{6422}{1/6}\hfill

    \benchmark{Offices}{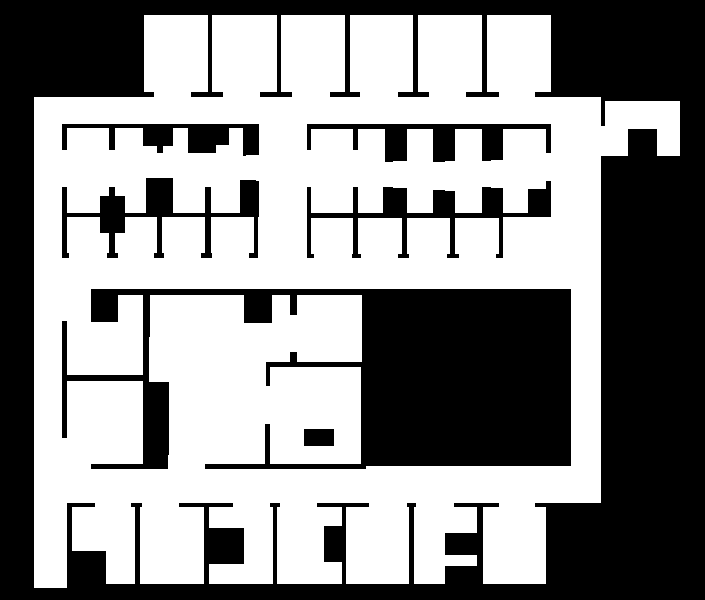}{1494}{2419}{29962}{25660}{0.09}\hfill
    \benchmark{Open}{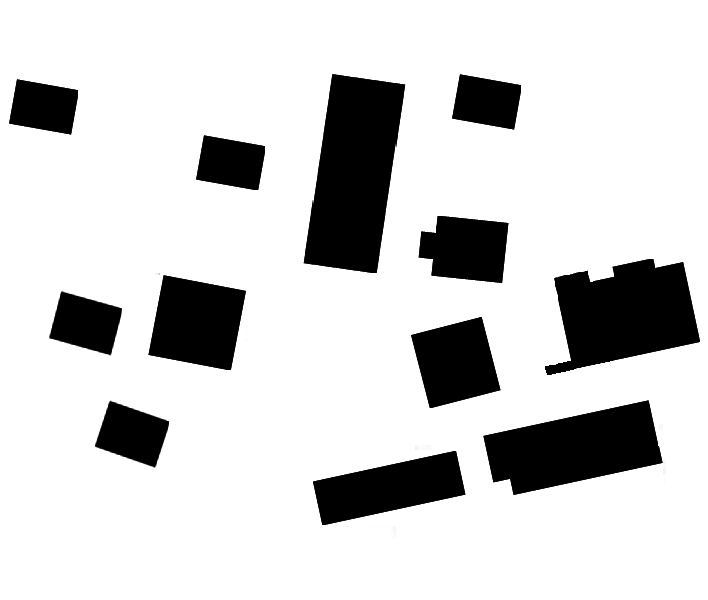}{2205}{4107}{39310}{39299}{0.08}
    \caption{Benchmarks. The four maps used to obtain topological graphs. Obstacles are in black. For each map, we generate two topological graphs that share the same nodes and movement edges, but the communication edges correspond either to a distance-based communication (the range we used are depicted below the maps) or a LOS-based communication.}
    \label{fig:maps}
\end{figure*}

The experiments were carried out on 4 different benchmark maps depicted in
Figure~\ref{fig:maps}. Both \emph{Coast} and \emph{Maze} come from the Moving AI
Lab benchmarks\footnote{\url{https://movingai.com/benchmarks/mapf/index.html}
(\url{w_woundedcoast.map} and \url{maze-32-32-2.map})}. \emph{Coast} is
extracted from Dragon Age 2, in the same spirit as in
\cite{DBLP:journals/tciaig/Sturtevant12}. Both \emph{Offices} and \emph{Open}
maps were used for experimental analysis in \cite{Hollinger:12:MCPC} and
\cite{Tateo:18:MCPP}. \emph{Offices} is a map of the SDR offices from the Radish
data set \cite{Radish}. \emph{Open} is a map of the McKenna MOUT site.

We discretized the maps as follows. The movement edges follow an 8-way grid
(i.e. the agents can move in the 8 directions). Concerning communications, we
adopted two practical settings, the distance-based and the line-of-sight-based (LOS)
ones. Both are standard; see \textit{e.g.} \cite{abb-is17}. We thus obtain $4 \times 2$ topological graphs.
\begin{itemize}
    \item In the \emph{distance-based} communication,  an agent communicates
         with another one up to a certain maximal distance, called the range
         (e.g. as in Wi-Fi); the range is displayed below the maps in
         Figure~\ref{fig:maps};
    \item In the \emph{LOS}-based communication, an agent communicates with
        agents that are in line of sight, (do not pass through obstacles).
\end{itemize}

Note that both communication models are used in applications. LOS is particularly
interesting to obtain a conservative model without false negatives, that is,
assumed communication links are likely to exist in the real-world. In contrast, 
in the distance-based communication model, some obstacles can prevent communication
between two locations or lower its quality. See the discussion in~\cite{abb-is17}.

\subsubsection{Methodology}

The algorithms were implemented single-threaded in Java. The experiments were
done sequentially and the time of initialization of the algorithms (e.g. parsing
of the graph, etc) was not counted in its execution time. These experiments were
done on an Intel Xeon W-2104 CPU at 3.20GHz with 16 GB of memory.

We compare the three main algorithms ($\CCBSSO$, $\CCBSNSO$, $\CCBSN$) and A*
algorithm~\cite{Hart:68:ASTAR} with the operator decomposition (OD)
optimization~\cite{Standley:10:FOS},
which consists in moving a single agent per step.
We run the algorithms with 2 to 10 agents
and then from 10 to 50 agents (10, 15, \dots, 45, 50), on the 8 topological
graphs obtained from the 4 maps listed above with either a \emph{distance-based}
or \emph{LOS-based} communication on 100 instances with a time out of 30 seconds
to solve 10 instances.

The success rate of an algorithm on these benchmarks is defined as the percentage
of instances for which it found an execution in the allocated time.
These executions are optimal for $\CCBS$ and~$\CCBSN$ but not necessarily
optimal for~$\CCBSSO$.

In Figure~\ref{fig:radius-sucess}, we report the success rates of $\CCBSN$, 
$\CCBSNSO$, $\CCBSSO$ and the A$^*$ algorithm on the 4 maps with a \emph{distance-based} communication. In
Figure~\ref{fig:sight-sucess}, we report the same experiments with the
\emph{LOS-based} communication.

\subsection{Results}

\newcommand{\mycustomlegend}[1]{
    \tikz[baseline=-1mm,xscale=0.5]{
        \draw[line width=0.8mm, #1] (0, 0) -- (1, 0);
 }
}
\newcommand{\CCBSNlegend}{\mycustomlegend{red, dash pattern={on 2pt off 2pt}}}
\newcommand{\CCBSNSOlegend}{\mycustomlegend{cyan!30!gray}}
\newcommand{\CCBSSOlegend}{
    \mycustomlegend{green!70!black,dash pattern={on 2pt off 2pt on 4pt off 2pt}}
}
\newcommand{\DFSlegend}{
    \mycustomlegend{black!70!black,dash pattern={on 4pt off 4pt}}
}
\newcommand\x{0.6}
\newcommand{\displaysuccessratediagrams}[1]{
	\begin{tikzpicture}
	\node at (0, 0) {\includegraphics[width=\x\textwidth]{figures/coast_#1.pdf}};
	\node at (\x\textwidth, 0) {\includegraphics[width=\x\textwidth]{figures/maze_#1.pdf}};
	\node at (0, -3.3) {\includegraphics[width=\x\textwidth]{figures/offices_#1.pdf}};
	\node at (\x\textwidth, -3.3) {\includegraphics[width=\x\textwidth]{figures/open_#1.pdf}};
	\node at (1.2, 1) {coast};
	\node at (34+\x\textwidth, 1) {maze};
	\node at (1.2, -3.3+1.05) {offices};
	\node at (34+\x\textwidth, -3.3+1) {open};
	\node at (0, -5) {\scriptsize Number of agents};
	\node at (\x\textwidth, -5) {\scriptsize Number of agents};
	\node[rotate=90] at (-2.2,0) {\scriptsize Success rate};
	\node[rotate=90] at (-2.2,-3.3) {\scriptsize Success rate};
    \end{tikzpicture}
}
\begin{figure*}
    \centering
    \fbox{
    	\scriptsize
    	\begin{tabular}{ll}
    		\CCBSSOlegend & $\CCBSSO$ (incomplete) \\[-0.5mm]
    		\CCBSNSOlegend & $\CCBSNSO$ \\[-0.5mm]
    		\CCBSNlegend & $\CCBSN$ \\[-0.5mm]
    		\DFSlegend & A*
    \end{tabular}}

    	\centering
    \begin{subfigure}{\x\textwidth}
		\displaysuccessratediagrams{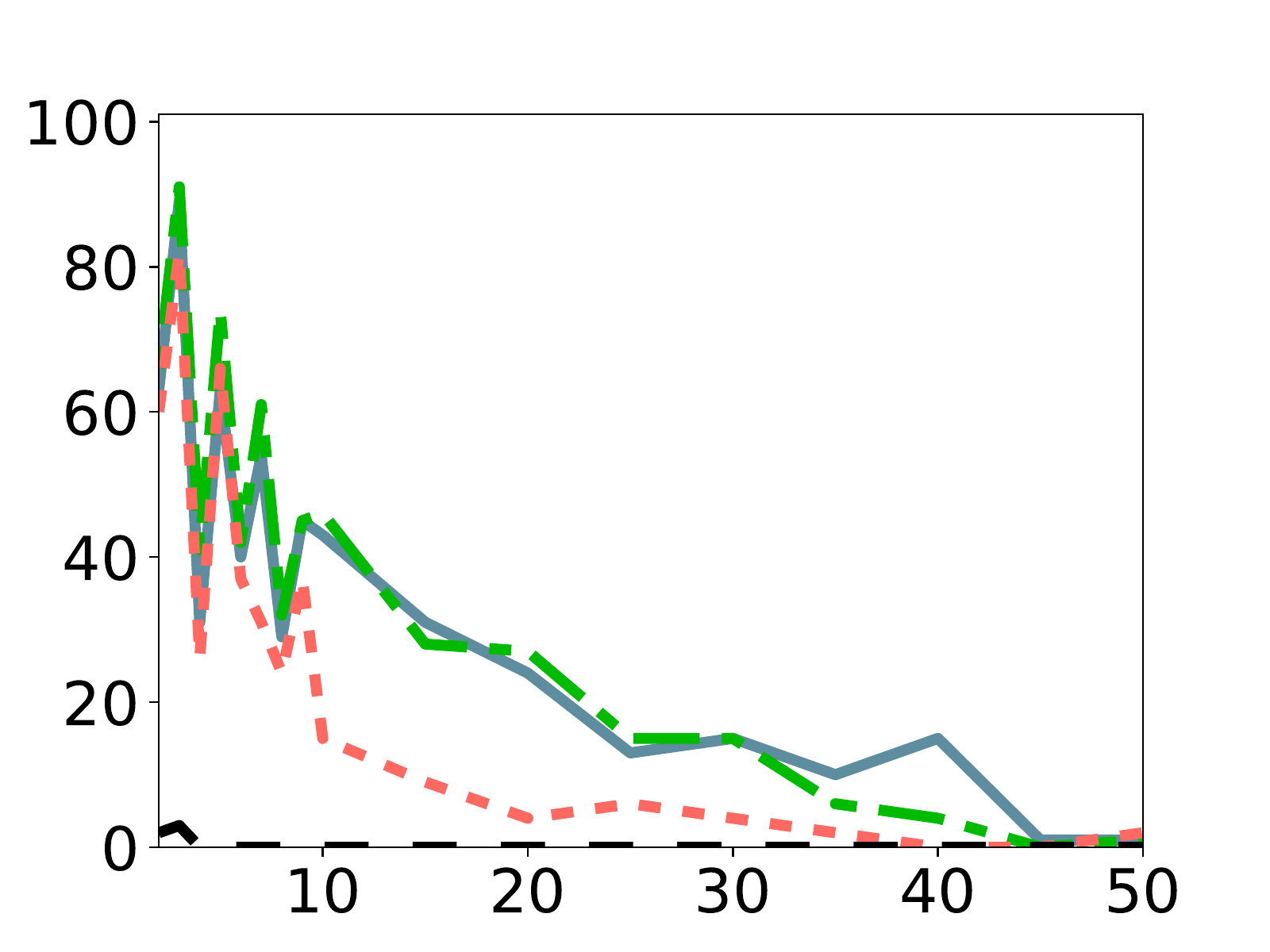}
        \caption{Distance-based communication.}
        \label{fig:radius-sucess}
    \end{subfigure}
    \hfill
    
    \centering
    \begin{subfigure}{\x\textwidth}
        \displaysuccessratediagrams{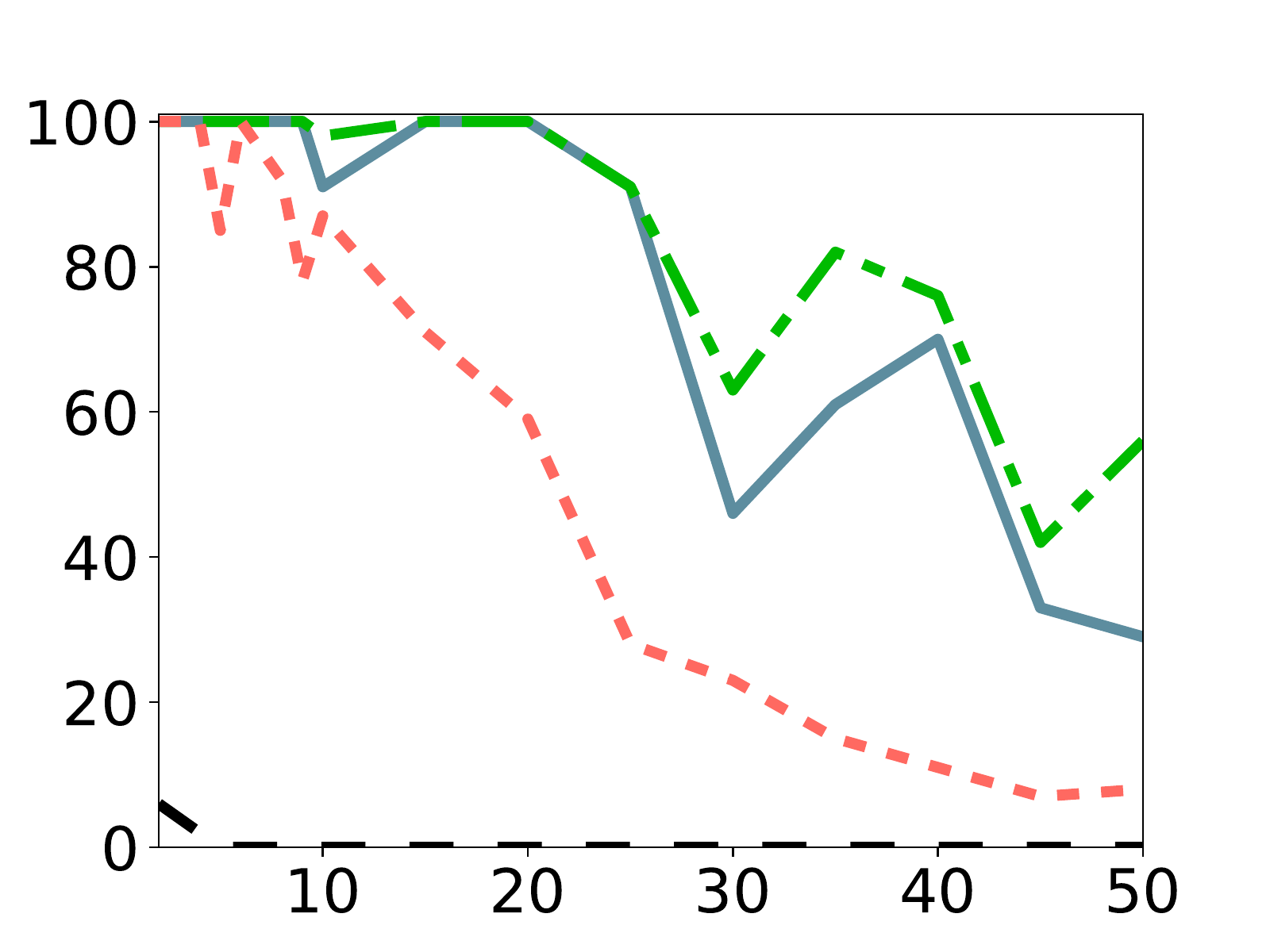}
        \caption{LOS-based communication.}
        \label{fig:sight-sucess}
    \end{subfigure}
    %
        
     %

    \caption{Success Rate on the benchmark Coast, Maze, Offices and Open with
    the two definitions for communication edges. $\CCBSNSO$ generally
    outperforms $\CCBSN$. $\CCBSSO$ generally slightly outperforms $\CCBSNSO$.
    Our algorithms outperform A*.
    }
    \label{fig:total-sucess}
\end{figure*}

\subsubsection{A$^*$ Algorithm}

Our optimal algorithms $\CCBS$ and $\CCBSN$ outperform A$^*$ by an order of magnitude on all maps. We believe
that the main reason for the inefficiency of~A$^*$ is the branching factor which is at most 9 for moving a
single agent (8 directions and idle), but $9^n$ for a joint move of $n$ agents.
Notice that the performance of A$^*$ was slightly better in Maze maps with LOS
communication which has a smaller branching factor. Another reason could be that
the makespan objective prevents A$^*$ from distinguishing better executions. In
fact, a long path taken by a single agent can shadow improvements in the paths
of other agents. The A$^*$ algorithm is our implementation of the DFS-based
algorithm of~\cite{Tateo:18:MCPP}.



\subsubsection{$\CCBSNSObold$ v.s. $\CCBSNbold$}

One can observe, first, that $\CCBSNSO$ outperforms $\CCBSN$ in all benchmarks.
Indeed, the addition of the \Self and \Other strategies allows $\CCBSNSO$ to
gain $16\%$ of success rate in average. Both algorithms have a similar behavior
on some maps as the Maze and Coast with \emph{distance-based} communication.
However, on the map Open, $\CCBSNSO$ performs better.

A second observation is the difference in success rates depending on the type
of communication. In our experiments, the algorithms performed generally better
when communication was LOS-based rather than distance-based.
In particular, on the Maze map, almost no
execution were generated past $20$ agents with the \emph{distance-based}
communication, while with the \emph{LOS-based} communication, $\CCBSNSO$ is
still above $40\%$ of success with $50$ agents.

\subsubsection{$\CCBSNSObold$ v.s. $\CCBSSObold$}

The sub-optimal $\CCBSSO$ is $4\%$ better in
total average of success rate than $\CCBSNSO$. Interestingly, over 13600 executions, $\CCBSNSO$ and $\CCBSSO$
compute the same result except for 4 cases. For our benchmarks, the non-optimal
algorithm $\CCBSSO$ found optimal executions in $99.97\%$ cases.
This indicates that despite its incompleteness (Theorem~\ref{thoerem:SelfOtherInComplete}),
$\CCBSSO$ seems suitable in practice.

\subsubsection{Size of the constraint tree}
The strategies $\Self,\Other$ used in $\CCBS$ lead to a larger
branching factor compared to~$\CCBSN$. In fact, rather than selecting a disconnected agent and forbidding
its current location, these strategies select a disconnected agent and a candidate vertex. Figure~\ref{fig:tree} shows a comparison of the number of nodes of the
constraint
trees generated by both algorithms. Despite the large branching factor, $\CCBS$ often quickly
finds a solution which means that the depth of the conflict tree is small; while $\CCBSN$
generates significantly larger trees.

\begin{figure}[ht]
    \centering
    \begin{tikzpicture}
        \node at (0, 0) {\includegraphics[width=0.6\textwidth]{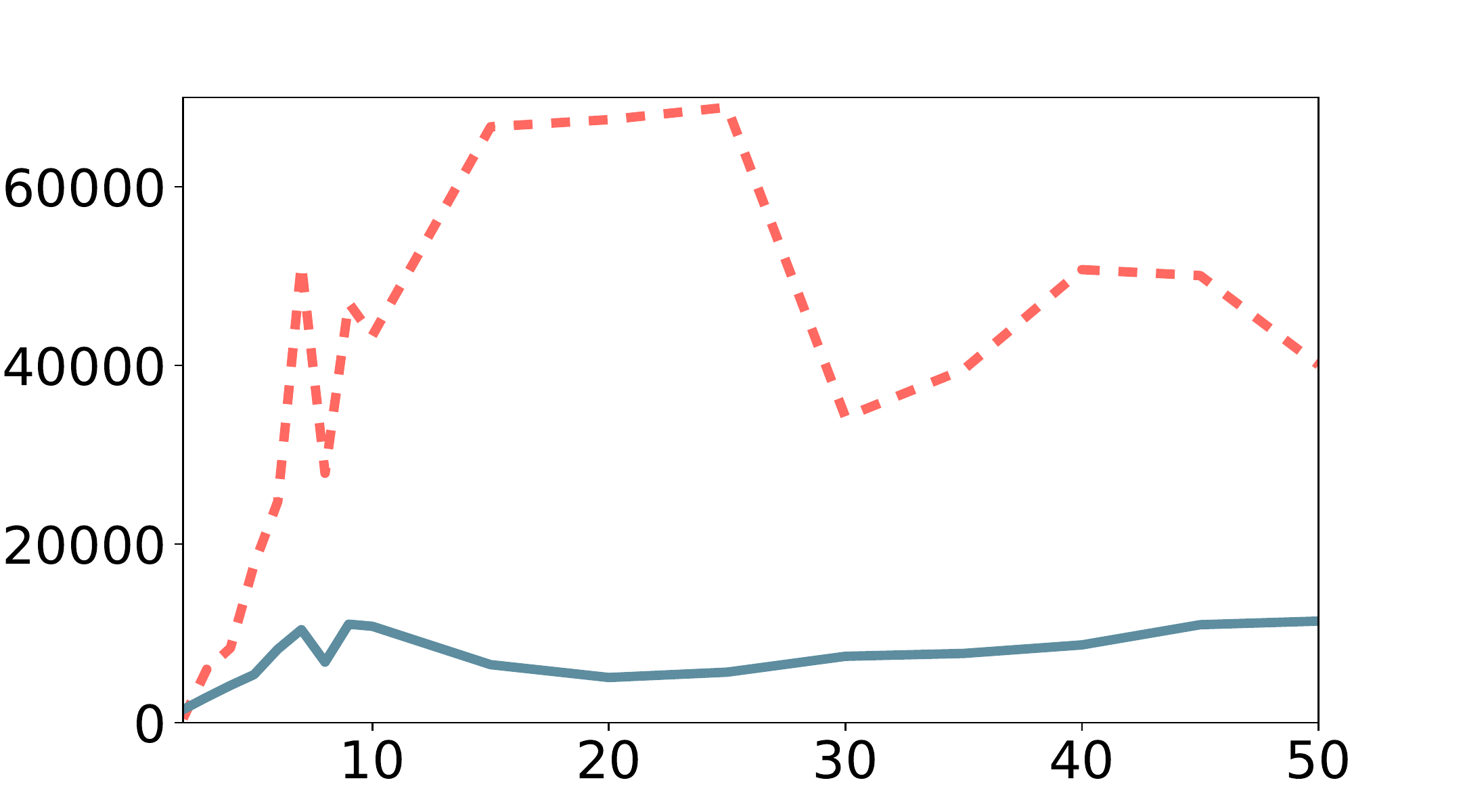}};
	    \node at (0, -2) {\scriptsize Number of agents};
	    \node[rotate=90] at (-3.9,0) {\scriptsize Number of generated nodes};
    \end{tikzpicture}
    \scriptsize
    \begin{tabular}{ll}
        \CCBSNSOlegend & $\CCBSNSO$ \\[-0.5mm]
        \CCBSNlegend & $\CCBSN$
    \end{tabular}
    \caption{Average number of generated nodes for successful executions. The
    number of nodes generated by $\CCBSNSO$ is lower by an order of magnitude
    compared to $\CCBSN$.}
    \label{fig:tree}
\end{figure}

\subsection{Partial and Selective Splitting}

Now, we discuss an optimization used to lower the memory usage of
the algorithm $\CCBS$.

In a variant of the A* algorithm from \cite{Yoshizumi:2000:PEL} and its enhanced
version~\cite{Felner:12:PSN}, the authors introduce an optimization which
consists in partially generating the children of nodes based on their costs.
When a node $\vNode$ is chosen for splitting, one can partially split it (that
is, generate only some of its child nodes) and put the node back in OPEN so that
it is split again later (to generate the rest of its child nodes).

In our setting, if the child node \Neg generated for agent~$\vAgent$ has an
execution strictly longer than that of the parent node, then we know that the
executions of agent~$\vAgent$ at all \Self and \Other nodes that constrain
agent~$\vAgent$ are longer as well. We thus suggest the following optimization.
Given a conflict at time $t$, if the path of an agent $\vAgent$ with the new
constraint $\langle \vAgent, {\vExec_\vAgent[t]}, \vTime, \bot \rangle$ is
longer than its previous path then we partially split the node by omitting the
generation of all child nodes in which agent~$\vAgent$ is constrained. The
current node is put back in $\aOPEN$ by incrementing its cost.

This is particularly useful in our case given the large amount of \Self and
\Other constraints created. In our experiments, this partial and selective
splitting optimization reduced the number of created nodes by $9\%$ in average
over all benchmarks with slightly better success rate (the average success rate
is $59\%$ with optimization and $57\%$ without).


\section{Conclusion}
\label{sec:concl}
We presented the optimal algorithm Connected-Conflict-based Search (\CCBS)
(Algorithm~\ref{alg:CBS}) for solving the Connected Multi-Agent Path Finding
problem (Definition~\ref{def:pb:reach}). We then investigated the impact of the
strategies \Self and \Other. Omitting these yields $\CCBSN$ that uses only $\Neg$,
which is still optimal but has worse performance, although it does outperform A*.
Using $\Self$ and $\Other$ but not $\Neg$ yields $\CCBSSO$  that 
is incomplete (Theorem~\ref{thoerem:SelfOtherInComplete})
but produced optimal results in most cases in our experiments,
and had a better success rate than $\CCBSNSO$.


Note that \Self and \Other strategies have tendency to
increase the branching factor. Fortunately, these constraints also significantly
improved the execution time in practice. In contrast, in some work, e.g. the
bypass optimization~\cite{Boyarski:15:CBSBP} for CBS or the operator
decomposition for A*~\cite{Standley:10:FOS}, the authors aim at reducing the
branching factor to improve performance. This is why, while our specific strategies
adding positive constraints might not be fitted for the original MAPF
problem, finding new types of constraints might improve the
performance both for CBS and CCBS.
As shown in Figure~\ref{fig:tree}, strategies \Self and \Other reduce the number of generated nodes.

Interestingly, our Algorithm~\ref{alg:CBS} can be easily tuned to handle \emph{both} collision and communication conflicts. 
For this, Line~7 needs to be modified to detect the
two different types of conflicts. If the conflict is a collision, apply
the original strategy of CBS and jump to Line~11; otherwise, pursue from
Line~8. Our algorithm can also be adapted for other definitions of execution costs, such
as the sum of the path lengths. The experimental study of those extensions is
left for future work. 

Several optimizations for CBS are worth to be extended for CCBS:
\begin{inparaenum}[i)]
	\item grouping agents as meta-agents (MA-CBS \cite{Sharon:12:MACBS}),
	\item prioritizing conflicts (ICBS \cite{Boyarski:15:ICBS}),
	\item and adding an heuristic to the search (CBS-h \cite{Felner:18:CBSh}).
\end{inparaenum}

\newpage
\bibliographystyle{alpha}
\bibliography{main}


\end{document}